\documentclass[twoside,11pt]{article}

\usepackage{jmlr2e}
\usepackage[english]{babel}
\usepackage{amsfonts}
\usepackage{amsmath}
\usepackage{graphicx}
\usepackage{amsfonts}
\usepackage{latexsym}
\usepackage{graphicx}
\usepackage{cases}
\usepackage[mathscr]{euscript}
\usepackage{xcolor}
\usepackage{colortbl}
\usepackage{thmtools}
\usepackage{thm-restate}

\usepackage{MnSymbol}
\DeclareMathAlphabet\mathbb{U}{msb}{m}{n}
\usepackage{xpatch}

\definecolor{Gray}{gray}{0.85}
\newcolumntype{g}{>{\columncolor{Gray}}c}

\hypersetup{
  colorlinks   = true, 
  urlcolor     = blue, 
  linkcolor    = blue, 
  citecolor   = blue 
}

\def\Rset{\mathbb{R}}

\DeclareMathOperator*{\E}{\mathbb{E}}

\DeclareMathOperator{\sign}{sign}

\newcommand{\nrm}[1]{{\left\vert\kern-0.25ex\left\vert\kern-0.25ex\left\vert #1 
    \right\vert\kern-0.25ex\right\vert\kern-0.25ex\right\vert}}

\newcommand{\cA}{\mathcal{A}}

\newcommand{\cD}{\mathcal{D}}
\newcommand{\sF}{\mathcal{F}}

\newcommand{\cS}{\mathcal{S}}
\newcommand{\cR}{\mathcal{R}}

\newcommand{\sD}{{\mathscr D}}
\newcommand{\sG}{{\mathscr G}}

\newcommand{\sS}{{\mathscr S}}

\newcommand{\bA}{{\mathbf A}}

\newcommand{\bI}{{\mathbf I}}

\newcommand{\bM}{{\mathbf M}}

\newcommand{\bU}{{\mathbf U}}

\newcommand{\bX}{{\mathbf X}}

\newcommand{\be}{{\mathbf e}}

\newcommand{\bu}{{\mathbf u}}
\newcommand{\bv}{{\mathbf v}}
\newcommand{\bw}{{\mathbf w}}
\newcommand{\bx}{{\mathbf x}}
\newcommand{\bz}{{\mathbf z}}
\newcommand{\bsigma}{{\boldsymbol \sigma}}

\newcommand{\R}{\mathfrak R}

\newcommand{\one}{\mathbf{1}}

\newcommand{\du}[1]{{#1}^*}
\newcommand{\h}{\widehat}

\newcommand{\la}{\langle}
\newcommand{\ra}{\rangle}
\newcommand{\set}[2][]{#1 \{ #2 #1 \} }

\newcommand{\ignore}[1]{}

\begin{document}

\title{On the Rademacher Complexity of Linear Hypothesis Sets}

\author{\name Pranjal Awasthi \email pranjalawasthi@google.com \\
       \addr Google Research \& Rutgers University\\
       New York, NY 10011, USA
       \AND
       \name Natalie Frank \email nf1066@nyu.edu \\
       \addr Department of Mathematics\\
       \addr Courant Institute of Mathematical Sciences\\
       New York, NY 10012, USA
       \AND
       \name Mehryar Mohri \email mohri@google.com \\
       \addr Google Research \&\\
       \addr Courant Institute of Mathematical Sciences\\
       New York, NY 10011, USA}

\editor{TBD}

\maketitle

\begin{abstract}
  Linear predictors form a rich class of hypotheses used in a variety
  of learning algorithms. We present a tight analysis of the empirical
  Rademacher complexity of the family of linear hypothesis classes
  with weight vectors bounded in $\ell_p$-norm for any $p \geq
  1$. This provides a tight analysis of generalization using these
  hypothesis sets and helps derive sharp data-dependent learning
  guarantees.  We give both upper and lower bounds on the Rademacher
  complexity of these families and show that our bounds improve upon
  or match existing bounds, which are known only for
  $1 \leq p \leq 2$.

\end{abstract}

\section{Introduction}

Linear predictors form a rich class of hypotheses used in a variety
of learning algorithms, including SVM \citep{CortesVapnik1995},
logistic regression or conditional maximum entropy models
\citep{BergerDellaPietraDellaPietra1996}, ridge regression
\citep{HoerlKennard1970}, and Lasso \citep{Tibshirani1996}.

Different regularizations or $\ell_p$-norm conditions are used to
constrain the family of linear predictors. This short note gives a
sharp analysis of the generalization properties of linear predictors
for arbitrary $\ell_p$-norm upper bound constraints. To do so, we give
tight upper bounds on the empirical Rademacher complexity of these
hypothesis sets which we show are matched by lower bounds, modulo some
constants.

The notion of Rademacher complexity is a general complexity measure
used to derive sharp data-dependent learning guarantees for different
hypothesis sets, including margin bounds, which are key in the
analysis of generalization for classification
\citep{KoltchinskiiPanchenko2002,BartlettMendelson2002,
  MohriRostamizadehTalwalkar2018}. There are known upper bounds on the
Rademacher complexity of linear hypothesis sets for some values of
$p$, including $p = 1$ or $p = 2$
\citep{BartlettMendelson2002,MohriRostamizadehTalwalkar2018}, as well
as $1 < p < 2$ \citep{KakadeSridharantTewari2008}. Our upper bounds on
the empirical Rademacher complexity are tighter than those known for
$1 \leq p < 2$ and match the existing one for $p = 2$. We further give
upper bounds on the Rademacher complexity for other values of $p$
($p > 2$). Our upper bounds are expressed in terms of
$\| \bX^\top \|_{2, p^*}$, where $\bX$ is the matrix whose columns are
the sample points and where $p^*$ conjugate number associated to
$p$. We give matching lower bounds in terms of the same quantity for
all values of $p$, which suggest the key role played by this
quantity in the analysis of complexity.

Much of the results presented here already appeared in 
\citep{AwasthiFrankMohri2020}, in the context of the analysis 
of adversarial Rademacher complexity. Here, we
present a more self-contained and detailed analysis, including the
statement and proof of lower bounds.
In Section~\ref{sec:preliminaries}, we introduce some preliminary
definitions and notation. We present our new upper and lower bounds on
the Rademacher complexity of linear hypothesis sets in
Section~\ref{sec:linear_function_classes} (Theorem~\ref{th:main} and
Theorem~\ref{th:lowerbound}). The proof of the upper bounds is given in
Appendix~\ref{app:main} and that of the lower bounds in
Appendix~\ref{app:lowerbound}. Lastly, in Appendix~\ref{app:compare}
we give a detailed analysis of how our bounds improve upon existing
ones.

\section{Preliminaries}
\label{sec:preliminaries}

We will denote vectors as lowercase bold letters (e.g., $\bx$) and
matrices as uppercase bold (e.g., $\bX$). The all-ones vector is
denote by $\one$. The H\"older conjugate of $p \geq 1$ is denoted by
$p^*$. For a matrix $\bM$, the $(p, q)$-\emph{group norm} is defined
as the $q$-norm of the $p$-norm of the columns of $\bM$, that is
$\| \bM \|_{p, q} = \| (\|\bM_1\|_1, \ldots, \|\bM_d\|_p) \|_q$, where 
$\bM_i$s are the columns of $\bM$. 

Let $\sF$ be a family of functions mapping from $\Rset^d$ to
$\Rset$. Then, the \emph{empirical Rademacher complexity} of $\sF$ for
a sample $\sS = (\bx_1, \ldots \bx_m)$, is defined by
\begin{align}
\label{eq:erc_def}
\h \R_\sS(\sF)
= \E_\bsigma \left[ \sup_{f\in \sF}\frac 1 m \sum_{i=1}^m \sigma_i f(\bx_i) \right],
\end{align}
where $\bsigma = (\sigma_1, \ldots, \sigma_m)$ is a vector of i.i.d.\
Rademacher variables, that is independent uniform random variables
taking values in $\set{-1, +1}$. The \emph{Rademacher complexity} of
$\sF$, $\R_m(\sF)$, is defined as the expectation of this quantity:
$\R_m(\sF) = \E_{\sS \sim \sD^m}[\h \R_\sS(\sF)]$, where $\sD$ is a
distribution over the input space $\Rset^d$.  The empirical Rademacher
complexity is a key data-dependent complexity measure. For a family of
functions $\sF$ taking values in $[0, 1]$, the following learning
guarantee holds: for any $\delta > 0$, with probability at least
$1 - \delta$ over the draw of a sample $S \sim \sD^m$, the following
inequality holds for all $f \in \sF$
\citep{MohriRostamizadehTalwalkar2018}:
\[
\E_{x \sim \sD}[f(x)] \leq \E_{x \sim \sS}[f(x)] + 2 \h \R_\sS(\sF) +
3 \sqrt{\frac{\log \frac{2}{\delta}}{2m}},
\]
where we denote by $\E_{x \sim \sS}[f(x)]$ the empirical average of 
$f$, that is $\E_{x \sim \sS}[f(x)] = \frac{1}{m} \sum_{i = 1}^m
f(x_i)$. A similar inequality holds for the average Rademacher
complexity $\R_m(\sF_p)=\E_{\sS\sim \sD^m} [\h \R_\sS(\sF)]$:
\[
\E_{x \sim \sD}[f(x)] \leq \E_{x \sim \sS}[f(x)] + 2 \R_m(\sF) +
\sqrt{\frac{\log \frac{1}{\delta}}{2m}}.
\]
An important application of these bounds is the derivation of 
margin bounds which are crucial in the analysis of classification.
Fix $\rho > 0$. Then, for any $\delta > 0$, with probability at least
$1 - \delta$ over the draw of a sample $S \sim \sD^m$, the following
inequality holds for all $f \in \sF$
\citep{KoltchinskiiPanchenko2002,MohriRostamizadehTalwalkar2018}:
\begin{align*}
\E_{(x, y) \sim \sD}[1_{y f(x) \leq 0}]
& \leq \E_{(x, y) \sim \sS}\left[\min \left(1, \Big( 1 - \tfrac{y f(x)}{\rho}
  \Big)_+  \right) \right] + \frac {2} \rho \h
\R_\sS(\sF) + 3\sqrt{\frac{\log \frac 2\delta}{2m}}\\
& \leq \frac{1}{m} \sum_{i = 1}^m 1_{y_i f(x_i) \leq \rho} + \frac {2} \rho \h
\R_\sS(\sF) + 3\sqrt{\frac{\log \frac 2\delta}{2m}}.
\end{align*}
Finer margin guarantees were recently presented by
\citet{CortesMohriSuresh2020} in terms of Rademacher complexity
and other complexity measures.
Furthermore, the Rademacher complexity of a hypothesis set
also appears as a lower bound in generalization. As an example,
for a symmetric family of functions $\sG$ taking values in $[-1, +1]$,
the following holds \citep{VaartWellner1996}:
\[
\frac{1}{2} \left[ \R_m(\sG) - \frac{1}{\sqrt{m}} \right]
\leq \E_{\sS \sim \sD^m} \left[ \sup_{f \in \sG} \left| \E_{x \sim \sD}[f(x)] -  \E_{x \sim \sS}[f(x)] \right|
\right] \leq 2 \R_m(\sG).
\]

The hypothesis set we will analyze in this paper is that of 
linear predictors whose weight vector is bounded in $\ell_p$-norm:
\begin{align}
\label{eq:linear_function_class}
\sF_p = \set[\big]{\bx \mapsto \bw \cdot \bx \colon \| \bw \|_p \leq W }.
\end{align}

\section{Empirical Rademacher Complexity of Linear Hypothesis Sets}
\label{sec:linear_function_classes}


The main results of this note are the following upper and lower
bounds on the empirical Rademacher complexity of linear hypothesis
sets.
\begin{restatable}{theorem}{maintheorem}
\label{th:main}
Let
$\sF_p = \set{\bx \mapsto \bw \cdot \bx \, \colon \| \bw \|_p \leq W}$
be a family of linear functions defined over $\Rset^d$ with bounded
weight in $\ell_p$-norm. Then, the empirical Rademacher complexity of
$\sF_p$ for a sample $\sS = (\bx_1, \ldots, \bx_m)$ admits the
following upper bounds:
\begin{align}
\h \R_\sS(\sF_p) 
\leq
\begin{cases}
\frac W m\sqrt{{2\log(2d)}} \, \| {\bX^\top}\|_{2, p^*} & \text{if $p = 1$} \nonumber \\
\frac{\sqrt{2}W}{m} \Bigg[\frac{\Gamma \left( \tfrac{p^* + 1}{2}
  \right)}{\sqrt{\pi}} \Bigg]^{\frac{1}{p^*}} \!\! \| \bX^\top \|_{2, p^*} & \text{if $1 <p \le 2$} \\ 
\frac{W}{m}\| \bX^\top \|_{2, p^*}, & \text{if  $p \ge 2$} \nonumber
\end{cases}
\end{align}
where $\bX$ is the $d \times m$-matrix with $\bx_i$s as
columns: $\bX = [\bx_1 \,\ldots\, \bx_m]$.
Furthermore, the constant factor in the inequality
for the case $1 < p \leq 2$ can be bounded as follows:
\[
e^{-\frac 12}\sqrt{p^*}
\leq \sqrt 2\bigg[\frac{\Gamma( \tfrac{p^* + 1}{2} )}{\sqrt{\pi}} \bigg]^{\frac{1}{p^*}} 
\leq e^{-\frac 12} \sqrt{p^* + 1}.
\]
\end{restatable}
The proof is given in Appendix~\ref{app:main}. Both the statement of
the theorem and its proof first appeared in \citep{AwasthiFrankMohri2020}
in the context of the analysis of adversarial Rademacher
complexity. We present a self-contained analysis in this note to make
the results more easily accessible, as we believe these results are of
a wider interest.  The next theorem is new and provides a lower bound
for $\h \R_\sS(\sF_p)$ which, modulo a constant factor, matches the
upper bounds stated above.

\begin{restatable}{theorem}{lowerboundtheorem}
\label{th:lowerbound}
Let
$\sF_p = \set{\bx \mapsto \bw \cdot \bx \, \colon \| \bw \|_p \leq W}$
be a family of linear functions defined over $\Rset^d$ with bounded
weight in $\ell_p$-norm. Then, the empirical Rademacher complexity of
$\sF_p$ for a sample $\sS = (\bx_1, \ldots, \bx_m)$ admits the
following lower bound, where $\bX = [\bx_1 \,\ldots\, \bx_m]$:
\begin{align}
\h \R_\sS(\sF_p)
\geq
\frac { W}{\sqrt 2 m} \| \bX^\top \|_{2, p^*}.
\end{align}
\end{restatable}
This lower bound is in tight in terms of dependence on sample size $m$ and dimension $d$. The proof is given in Appendix~\ref{app:lowerbound}. The following
corollary presents somewhat looser upper bounds that may be more
convenient in various contexts, such as that of kernel-based
hypothesis sets. The corollary can be derived directly by combining
Theorem~\ref{th:main} and Proposition~\ref{prop:norm_ratio} (see Section~\ref{sec:comparison}). 

\begin{corollary}
  Let
  $\sF_p = \set{\bx \mapsto \bw \cdot \bx \, \colon \| \bw \|_p \leq
    W}$ be a family of linear functions defined over $\Rset^d$ with
  bounded weight in $\ell_p$-norm. Then, the empirical Rademacher
  complexity of $\sF_p$ for a sample $\sS = (\bx_1, \ldots, \bx_m)$
  admits the following
  upper bounds, where $\bX = [\bx_1 \,\ldots\, \bx_m]$:\\
\begin{align*}
\text{for } p = 1, \qquad
\h\R_\sS(\sF_p) & \leq \frac Wm \sqrt {{2\log(2d)}}\|\bX\|_{p^*, 2};\\
\text{for } 1 < p \leq 2, \qquad
  \h\R_\sS(\sF_p) & \leq  e^{-\frac 12}\sqrt{p^*+1}\frac Wm \|\bX^\top\|_{2, p^*}\\
  \h\R_\sS(\sF_p) & \leq e^{-\frac 12}\sqrt{p^*+1}\frac Wm
  \|\bX\|_{p^*,2};\\
\text{for } p \geq 2, \qquad
\h \R_\sS(\sF_p) & \leq \frac W m \|\bX^\top\|_{2, p^*}\\
\h \R_\sS(\sF_p) & \leq \frac {W \min(m,d)^{\frac 1 {p^*}-\frac 12}}m \|\bX\|_{p^*,2}.
\end{align*}
\end{corollary}

\subsection{Discussion}
\label{sec:discussion}

We now make a few remarks about Theorem~\ref{th:main} and present the
proof in Appendix~\ref{app:main}. The theorem states that for any data set, $\h \cR_\sS(\sF_p)$ is a constant times
$\frac 1m \|\bX^\top\|_{2, p^*}$. This is in contrast to the quantity $\|\bX^\top\|_{p^*, 2}$ that appears in the existing analysis available in the literature for linear hypothesis sets~\citep{KakadeSridharantTewari2008}. However, as we will soon see in Theorem~\ref{th:rc_linear_comparison} using $\|\bX^\top\|_{2, p^*}$ always leads to a better upper bound.

 Another interesting aspect of the upper bound is the dimension
 dependence of the constant in front of $\|\bX^\top\|_{2, p^*}$. This
 constant is independent of dimension only for $p > 1$. For $p = 1$,
 the $\sqrt{\log (d)}$ dependence on dimension is tight, which can be
 seen from the correspondence tightness of the maximal inequality and
 thus that of Massart's inequality \citep{BoucheronLugosiMassart2013}.
 We also provide a simple example further illustrating this dependence
 in Appendix~\ref{app:sqrt_log_d}. This observation also explains why
 the constant for $p > 1$ approaches infinity as $p\to 1$: if we had
 that
\[\h \R_\sS(\sF_p)\leq c(p) \|\bX^\top\|_{2, p^*}\] for $p>1$, then by continuity 
\[\h \R_\sS(\sF_1)\leq \lim_{p\to 1} c(p) \|\bX^\top\|_{2,\infty}\]
If $c(p)$ were dimension independent and $\lim_{p\to \infty} c(p)$ were finite, then the constant for $p=1$ would be finite and dimension independent as well. Since we just showed that the constant for $p=1$ must have dimension dependence, we must have that $\lim_{p\to 1} c(p)=\infty$. This observation suggests that finding dimension-dependent constant for $1<p<2$ could greatly improve the upper bound of Theorem~\ref{th:main}. However, our example where the dimension dependence was tight for $p=1$ had $d=2^m$, which is unrealistic for most applications. It's possible that with some reasonable assumption on the relationship between $m$ and $d$, one could find a far better constant for $1<p<2$.

\subsection{Comparison with Previous Work}
\label{sec:comparison}

\begin{figure}[t]
\centering
\begin{tabular}{c@{\hspace{1.5cm}}c}
\includegraphics[scale=.3]{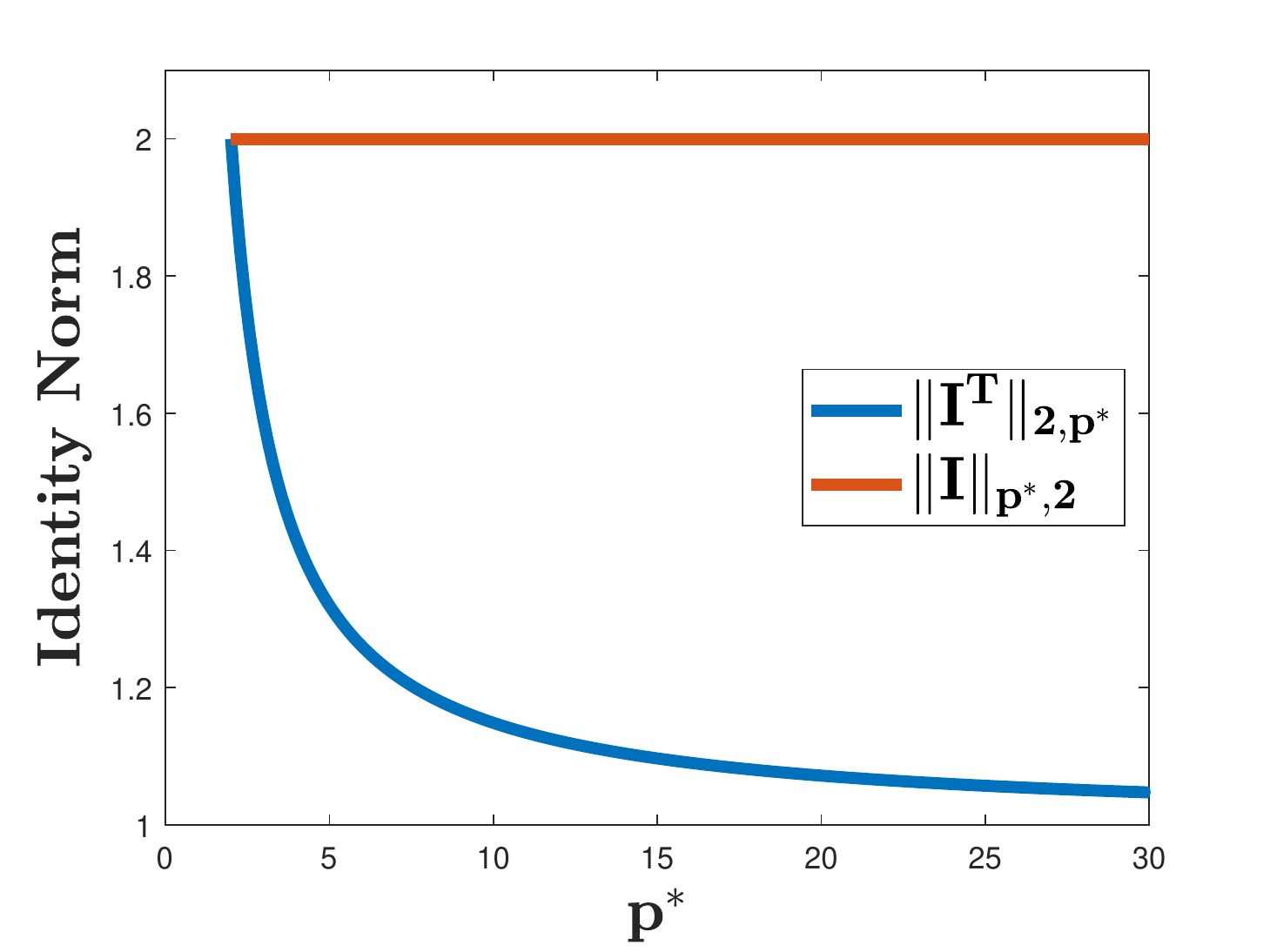}
& \includegraphics[scale=.3]{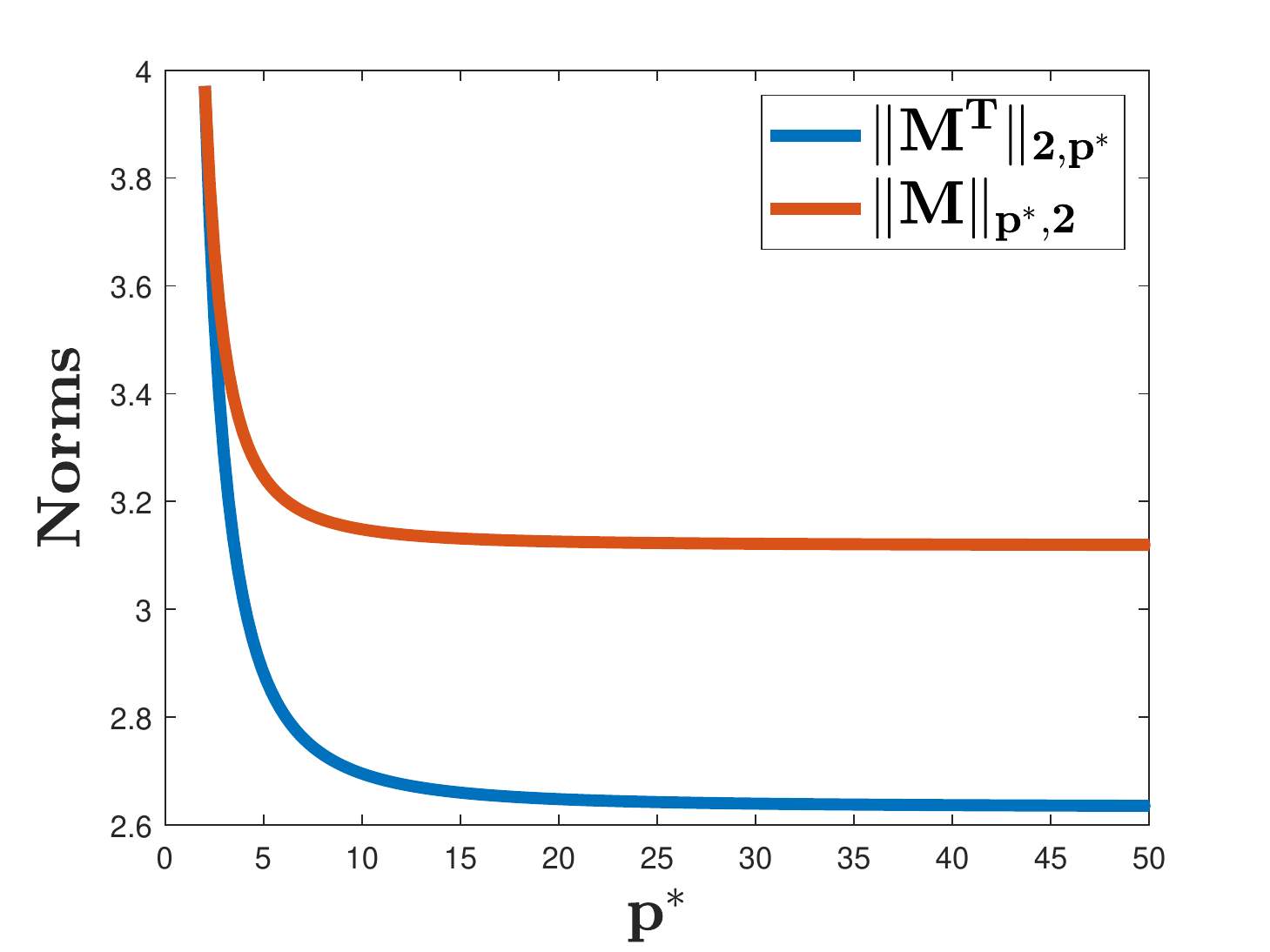}\\
(a) & (b)
\end{tabular}
\vskip -.1in
\caption{\label{fig:norms}(a) A plot comparing two norms of the
$4 \times 4$ identity matrix, $\|\bI^\top\|_{2, p^*}$ and
$\|\bI\|_{p^*,2}$; the lower bound on the ratio of the two norms
(\ref{eq:q_leq_p}) in Proposition~\ref{prop:norm_ratio} holds for
this matrix. (b) Same as (a), but for Gaussian matrices.}
\vskip -.15in
\end{figure}

We are not aware of any existing bound for the empirical Rademacher
complexity of linear hypothesis sets for $p > 2$ before this work. For
other values of $p$, the best existing upper bounds were given by
\cite{KakadeSridharantTewari2008} for $1 <p \leq 2$ and by
\cite{BartlettMendelson2001} (see also
\citep{MohriRostamizadehTalwalkar2018}) for $p = 1$:
\begin{align}
\label{eq:previous_linear_rc}
\h\R_\sS(\sF_p)\leq
\begin{cases}
W \sqrt{\frac {2\log(2d)}m}\| \bX^\top\|_{+\infty, +\infty} & \text{ if  $p=1$} \\ 
\frac W m \sqrt{p^* - 1} \|{\bX}\|_{p^*, 2} & \text{ if $1 < p \leq
  2$} 
\end{cases}
\end{align}
Our new upper bound coincides with (\ref{eq:previous_linear_rc}) when
$p=2$ and is strictly tighter otherwise. Readers familiar with
Rademacher complexity bounds for linear hypothesis sets will notice
that our bound in this case depends on the norm $\| \bX^\top \|_{2, p^*}$.
In contrast, the previously known bounds depend on $\| \bX \|_{p^*,
  2}$. In fact, one can show that the $\|\bX^\top\|_{2,p^*}$ is always
smaller than $\|\bX\|_{p^*, 2}$ for $p \in (1,2]$, that is
$p^*\geq 2$, as shown by the last inequality of \eqref{eq:q_leq_p} in
the following proposition.

\begin{restatable}{proposition}{propnormratioprop}
\label{prop:norm_ratio} 
Let $\bM$ be a $d\times m$ matrix.
If $q\leq p$, then 
\begin{equation}
\label{eq:q_leq_p}
\min(m,d)^{\frac 1p-\frac 1q} \|\bM^\top\|_{p,q} \leq \|\bM\|_{q,p} \leq \|\bM^\top\|_{p,q}\end{equation}
If $q\geq p$, then 
\begin{equation}
\label{eq:p_leq_q}
\min(m,d)^{\frac 1p-\frac 1q} \|\bM^\top\|_{p,q}\geq \|\bM\|_{q,p}\geq \|\bM^\top\|_{p,q}
\end{equation}
These bounds are tight.
\end{restatable}
The proof is presented in Appendix~\ref{sec:norm_comparison}. To
visualize the ratio between these two norms, we plot the two norms for
various values of $p^*$ in figure~\ref{fig:norms}.

For convenience, in the
discussion below, we set $c_1(p)=\sqrt{p^*-1}$ and
$c_2(p)=\sqrt 2\big[\frac{\Gamma( \tfrac{p^* + 1}{2} )}{\sqrt{\pi}}
\big]^{\frac{1}{p^*}}$. Regarding the growth of the constant in our bound, Theorem~\ref{th:main} implies that as $p^* \to \infty$,
$c_2(p)$ grows asymptotically like $e^{-\frac 12}\sqrt {p^*} $. Furthermore, $c_2(p) \leq c_1(p)$ in the relevant
region~(See Appendix~\ref{sec:const_bound}). In Figure~\ref{fig:constants} we plot $c_1(p),c_2(p)$ and
the bounds on $c_2(p)$ to illustrate the growth rate of these constants with $p^*$.
\begin{figure}[t]
\centering
\begin{tabular}{cc}
\includegraphics[scale=.3]{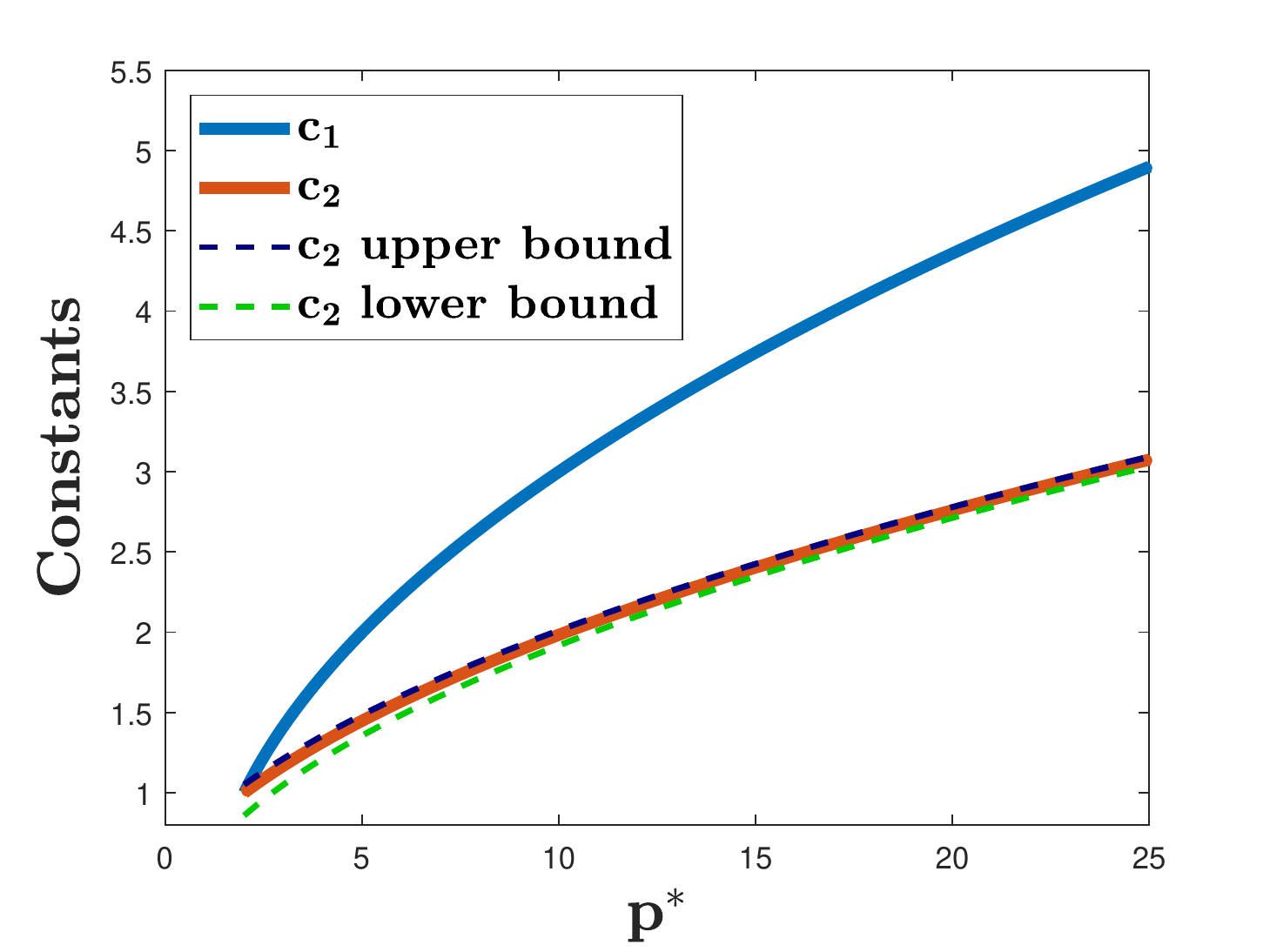}
\end{tabular}
\vskip -.15in
\caption{A plot of $c_1(p)$, $c_2(p)$, and the bounds from
Lemma~\ref{lemma:constant-comparison}. Note that $c_1(2) = c_2(2)$ and
that the upper and lower bounds on $c_2$ are tight.}
\vskip -.15in
\label{fig:constants}
\end{figure}

Proposition~\ref{prop:norm_ratio} and the inequality 
$c_2(p) \leq c_1(p)$ imply the following result.

\begin{restatable}{theorem}{normratioprop}
\label{th:rc_linear_comparison} 
For $p \leq 2$, the following inequality holds:
\[
\frac{\sqrt{2}W}{m} \bigg[\frac{\Gamma( \tfrac{p^* + 1}{2}
  )}{\sqrt{\pi}} \bigg]^{\frac{1}{p^*}} \|{\bX^\top}\|_{2, p^*}\leq \frac
W m \sqrt{p^*-1} \|{\bX}\|_{p^*,2}
\]
\end{restatable}
\noindent Thus, for $1 < p \leq 2$, the bound of Theorem~\ref{th:main} is
tighter than \eqref{eq:previous_linear_rc}.

\section{Conclusion}

We presented tight bounds on the empirical Rademacher complexity
of linear hypothesis sets constrained by an $\ell_p$-norm
bound on the weight vector. These bounds can be used to derive
sharp generalization guarantees for these hypothesis sets in a variety
of different contexts, by plugging them in existing 
Rademacher complexity learning bounds.
Our proofs and guarantees suggest an extension beyond $\ell_p$-norm
constrained hypothesis sets that we will discuss elsewhere.

\bibliography{lrad}

\begin{thebibliography}{17}
\providecommand{\natexlab}[1]{#1}
\providecommand{\url}[1]{\texttt{#1}}
\expandafter\ifx\csname urlstyle\endcsname\relax
  \providecommand{\doi}[1]{doi: #1}\else
  \providecommand{\doi}{doi: \begingroup \urlstyle{rm}\Url}\fi

\bibitem[Alzer(1997)]{Alzer1997}
Horst Alzer.
\newblock On some inequalities for the {G}amma and {P}si functions.
\newblock \emph{Math. Comput.}, 66\penalty0 (217):\penalty0 373--389, 1997.

\bibitem[Awasthi et~al.(2020)Awasthi, Frank, and Mohri]{AwasthiFrankMohri2020}
Pranjal Awasthi, Natalie Frank, and Mehryar Mohri.
\newblock Adversarial learning guarantees for linear hypotheses and neural
  networks.
\newblock In \emph{Proceedings of {ICML}}, 2020.

\bibitem[Bartlett and Mendelson(2001)]{BartlettMendelson2001}
Peter~L. Bartlett and Shahar Mendelson.
\newblock Rademacher and {G}aussian complexities: Risk bounds and structural
  results.
\newblock In \emph{Proceedings of {COLT}}, 2001.

\bibitem[Bartlett and Mendelson(2002)]{BartlettMendelson2002}
Peter~L. Bartlett and Shahar Mendelson.
\newblock Rademacher and {G}aussian complexities: Risk bounds and structural
  results.
\newblock \emph{Journal of Machine Learning Research}, 3, 2002.

\bibitem[Berger et~al.(1996)Berger, Pietra, and
  Pietra]{BergerDellaPietraDellaPietra1996}
Adam~L. Berger, Stephen~Della Pietra, and Vincent J.~Della Pietra.
\newblock A maximum entropy approach to natural language processing.
\newblock \emph{Comp. Linguistics}, 22\penalty0 (1), 1996.

\bibitem[Boucheron et~al.(2013)Boucheron, Lugosi, and
  Massart]{BoucheronLugosiMassart2013}
St{\'{e}}phane Boucheron, G{\'{a}}bor Lugosi, and Pascal Massart.
\newblock \emph{Concentration Inequalities - {A} Nonasymptotic Theory of
  Independence}.
\newblock Oxford University Press, 2013.

\bibitem[Cortes and Vapnik(1995)]{CortesVapnik1995}
Corinna Cortes and Vladimir Vapnik.
\newblock Support-vector networks.
\newblock \emph{Mach. Learn.}, 20\penalty0 (3):\penalty0 273--297, 1995.

\bibitem[Cortes et~al.(2020)Cortes, Mohri, and Suresh]{CortesMohriSuresh2020}
Corinna Cortes, Mehryar Mohri, and Ananda~Theertha Suresh.
\newblock Relative deviation margin bounds.
\newblock \emph{CoRR}, abs/2006.14950, 2020.

\bibitem[Haagerup(1981)]{Haagerup1981}
Uffe Haagerup.
\newblock The best constants in the {K}hintchine inequality.
\newblock \emph{Studia Mathematica}, 70:\penalty0 231--283, 1981.

\bibitem[Hoerl and Kennard(1970)]{HoerlKennard1970}
Arthur~E. Hoerl and Robert~W. Kennard.
\newblock Ridge regression: Biased estimation for nonorthogonal problems.
\newblock \emph{Technometrics}, 12\penalty0 (1):\penalty0 55--67, 1970.

\bibitem[Kakade et~al.(2008)Kakade, Sridharan, and
  Tewari]{KakadeSridharantTewari2008}
Sham~M. Kakade, Karthik Sridharan, and Ambuj Tewari.
\newblock On the complexity of linear prediction: Risk bounds, margin bounds,
  and regularization.
\newblock In \emph{Proceedings of NIPS}, pages 793--800, 2008.

\bibitem[Koltchinskii and Panchenko(2002)]{KoltchinskiiPanchenko2002}
Vladmir Koltchinskii and Dmitry Panchenko.
\newblock Empirical margin distributions and bounding the generalization error
  of combined classifiers.
\newblock \emph{Annals of Statistics}, 30, 2002.

\bibitem[Massart(2000)]{Massart2000}
Pascal Massart.
\newblock Some applications of concentration inequalities to statistics.
\newblock \emph{Annales de la Facult\'e des Sciences de Toulouse}, IX:\penalty0
  245--303, 2000.

\bibitem[Mohri et~al.(2018)Mohri, Rostamizadeh, and
  Talwalkar]{MohriRostamizadehTalwalkar2018}
Mehryar Mohri, Afshin Rostamizadeh, and Ameet Talwalkar.
\newblock \emph{Foundations of Machine Learning}.
\newblock The MIT Press, second edition, 2018.

\bibitem[Olver et~al.(2010)Olver, Lozier, Boisvert, and
  Clark]{OlverLozierBoisvertClark2010}
Frank W.~J. Olver, Daniel~W. Lozier, Ronald~F. Boisvert, and Charles~W. Clark.
\newblock \emph{The {NIST} Handbook of Mathematical Functions}.
\newblock Cambridge Univ. Press, 2010.

\bibitem[Tibshirani(1996)]{Tibshirani1996}
Robert Tibshirani.
\newblock Regression shrinkage and selection via the lasso.
\newblock \emph{Journal of the Royal Statistical Society. Series B},
  58\penalty0 (1):\penalty0 267--288, 1996.

\bibitem[van~der Vaart and Wellner(1996)]{VaartWellner1996}
Aad~W. van~der Vaart and Jon~A. Wellner.
\newblock \emph{Weak Convergence and Empirical Processes}.
\newblock Springer, 1996.

\end{thebibliography}

\renewcommand{\theHsection}{A\arabic{section}}
\clearpage
\appendix

\section{Proof of Theorem~\ref{th:main}}
\label{app:main}

In this section, we present the proof of Theorem~\ref{th:main}.

\maintheorem*
The proof proceeds in several steps. First, in
Appendix~\ref{sec:linear_p_1} we upper bound the Rademacher complexity
of $\sF_1$. Next, in Appendix~\ref{sec:new_linear_rc_proof}, we
establish the upper bound for $p > 1$. Lastly, in
Appendix~\ref{sec:const_bound}, we prove the inequalities for the
constant terms in the case $1 < p \leq 2$.

\subsection{Proof of the upper bound, case $p=1$}
\label{sec:linear_p_1}

The bound on the Rademacher complexity for $p=1$ was previously known
but we reproduce the proof of this theorem for completeness. We
closely follow the proof given in
\citep{MohriRostamizadehTalwalkar2018}.
\begin{proof}
For any $i \in [m]$, $x_{ij}$ denotes the $j$th component of
$\bx_i$.
\begin{align*}
\h\R_\sS(\sF_1)  
& = \frac{1}{m} \E_\bsigma \left[ \sup_{\| \bw \|_1 \leq W}
\bw \cdot \sum_{i = 1}^m \sigma_i \bx_i  \right]\\
& = \frac{W}{m} \E_\bsigma \left[ \Big\| \sum_{i = 1}^m
\sigma_i \bx_i  \Big\|_\infty \right] & \text{(by definition of the
dual norm)}\\
& = \frac{W}{m} \E_\bsigma \left[ \max_{j \in [d]} \left| \sum_{i = 1}^m
\sigma_i x_{ij}  \right|  \right] & \text{(by definition of $\| \cdot \|_\infty$)}\\
& = \frac{W}{m} \E_\bsigma \left[ \max_{j \in [d]} \max_{s \in
\set{-1, +1}} s \sum_{i = 1}^m
\sigma_i x_{ij}   \right] & \text{(by definition of $| \cdot |$)}\\
& = \frac{W}{m} \E_\bsigma \left[ \sup_{\bz \in \cA} \sum_{i = 1}^m
\sigma_i z_i  \right] ,
\end{align*}
where $\cA$ denotes the set of $d$ vectors
$\set{s (x_{1j}, \ldots, x_{mj})^\top \colon j \in [d], s \in
\set{-1, +1}}$.  For any $\bz \in A$, we have
$\| \bz \|_2 \leq \sup_{\bz\in A}
\|\bz\|_2=\|\bX^\top\|_{2,\infty}$. Further, $\cA$ contains
        at most $2d$ elements. Thus, by Massart's Lemma
        \citep{Massart2000, MohriRostamizadehTalwalkar2018}, 
\begin{align*}
\h\R_\sS(\sF_1)  
& \leq W  \|\bX^\top\|_{2, \infty} \frac{\sqrt{2 \log (2d)}}{m},
\end{align*}
which concludes the proof.
\end{proof}

\subsection{Proof of upper bound, case $p > 1$}
\label{sec:new_linear_rc_proof}

\begin{proof}
Here again, we use the shorthand
$\bu_\bsigma=\sum_{i = 1}^m \sigma_i\bx_i$. By definition of the dual
norm, we can write:
\begin{align*}
\h\R_\sS(\sF_p)
& = \frac{1}{m} \E_\bsigma \Bigg[ \sup_{\| \bw \|_p \leq W } \bw \cdot \sum_{i = 1}^m
\sigma_i \bx_i \Bigg]\\
& = \frac{W}{m} \E_\bsigma \big[ \| \bu_\bsigma \|_{p^*} \big] & (\text{dual norm property})\\
& \leq \frac{W}{m} \Big[ \E_\bsigma \big[ \| \bu_\bsigma \|_{p^*}^{p^*} \big]\Big]^{\frac{1}{p^*}}. &
(\text{Jensen's inequality, $p^* \in [1, +\infty)$})\\
& = \frac{W}{m} \Big[ \sum_{j = 1}^d
\E_\bsigma \big[ |\bu_{\bsigma, j}|^{p^*} \big] \Big]^{\frac{1}{p^*}}.
\end{align*}
Next, by Khintchine's inequality \citep{Haagerup1981}, the following
holds:
\begin{align*}
\E_\bsigma \big[ |\bu_{\bsigma, j}|^{p^*} \big] &\leq B_{p^*} \Big[ \sum_{i = 1}^m
x_{i, j}^2 \Big]^{\frac{p^*}{2}},
\end{align*}
where $B_{p^*} = 1$ for $p^* \in [1, 2]$ and
\begin{align*}
B_{p^*} &= 2^{\frac{p^*}{2}}\frac{\Gamma \big( \frac{p^* + 1}{2} \big)}{\sqrt{\pi}},
\end{align*}
for $p \in [2, +\infty)$.  This yields the following bound on the
Rademacher complexity:
\[
\h \R_S(\sF_p) \leq 
\begin{cases}
\frac{W}{m}\| \bX^\top \|_{2, p^*} & \text{if }p^* \in [1, 2], \\[.25cm]
\frac{\sqrt{2}W}{m} \bigg[\frac{\Gamma \big( \tfrac{p^* + 1}{2} \big)}{\sqrt{\pi}} \bigg]^{\frac{1}{p^*}} \| \bX^\top \|_{2, p^*} & \text{if } p^* \in [2, +\infty).
\end{cases}
\]
\end{proof}

\subsection{Bounding the Constant}
\label{sec:const_bound}

For convenience, set $c_2(p)\colon = \sqrt 2 \big(\frac{\Gamma(\frac{p^*+1} 2 )}{\sqrt \pi}\big)^\frac 1 {p^*}$.
We establish upper and lower bound on $c_2(p)$.  
\begin{lemma}\label{lemma:const_bound+}
\label{lemma:f2_bound}Let $c_2(p)=\sqrt 2 \big(\frac{\Gamma(\frac{p^*+1} 2 )}{\sqrt \pi}\big)^\frac 1 {p^*}$. Then the following inequalities hold:
\[
e^{-\frac 12}\sqrt{ p^*}\leq c_2(p)\leq  e^{-\frac 12} \sqrt {p^*+1}. 
\]
\end{lemma}
\begin{proof}
  For convenience, we set $q = p^*$, $f_1(q)=c_1(p)$, $f_2(q)=c_2(p)$.
  Next, we recall a useful inequality
  \citep{OlverLozierBoisvertClark2010} bounding the gamma function:
\begin{align}
1<(2\pi)^{-\frac 12}x^{\frac 12-x}e^x\Gamma(x)<e^{\frac 1 {12x}}.\label{eq:gamma_bound}
\end{align}

We start with the upper bound.
If we apply the right-hand side inequality of (\ref{eq:gamma_bound}) to $\Gamma(\frac {q+1}2)$ we get the following bound on $f_2(q)$:
\begin{equation}
\label{eq:f2_bound}
f_2(q)\leq 2^\frac 1 {2q} e^{-\frac 12} \sqrt{q+1} e^{-\frac 1{2q}+\frac 1 {6(q+1)q}}
\end{equation}
It is easy to verify that, 
\begin{equation}
\label{eq:upper_bound_factor}
2^\frac 1{2q}e^{-\frac 1 {2q} +\frac 1 {6q(q+1)}}=e^{\frac 1q(\frac {\ln 2-1}2+\frac 1 {6q(q+1)})}.
\end{equation}
Furthermore, the expression $(\frac {\ln 2-1}2+\frac 1 {6q(q+1)})$ decreases with increasing $q$. At $q=2$, it is negative, which implies that (\ref{eq:upper_bound_factor}) is less than 1 for $q\geq 2$. Hence
\begin{equation*}
f_2(q) \leq e^{-\frac 12} \sqrt{q+1} 
\end{equation*}

Next, we prove the lower bound.
Applying the lower bound of (\ref{eq:gamma_bound}) to $\Gamma(\frac{q+1}2)$ results in
\begin{equation*}
f_2(q)\geq e^{-\frac 12}\sqrt q \left(e^{-\frac 1{2q}(\log 2 -1)}\sqrt{1+\frac 1 q}\right).
\end{equation*} 
We will establish that $\Big(e^{-\frac 1{2q}(\log 2 -1)}\sqrt{1+\frac 1 q}\Big)\geq 1$, which will complete the proof of the lower bound. We prove this statement by showing that
\[
\left(e^{-\frac 1{2q}(\log 2 -1)}\sqrt{1+\frac 1 q}\right)^2
= e^{-\frac 1 q(\log 2 -1)} \left( 1+\frac 1q \right)\geq 1.
\]
By applying some elementary inequalities
\begin{align*}
e^{-\frac 1 q(\log 2 -1)}\left(1+\frac 1q\right)&\geq \left( \frac 1q(\log 2-1)+1\right) \left( 1 + \frac 1q \right) & (\text{using }e^x\geq 1+x)\\
&=1+\frac 1q \left( \log (2) -\frac{1-\log (2)}q\right)\\
&\geq 1
\end{align*}
The last inequality follows since $\Big( \log (2) -\frac{1-\log (2)}q\Big)$ increases with $q$, and is positive at $q = 2$. 
\end{proof}

\newpage

\section{Proof of Theorem~\ref{th:lowerbound}}
\label{app:lowerbound}

In this section, we prove the lower bound of
Theorem~\ref{th:lowerbound}.

\lowerboundtheorem*

\begin{proof}
For any vector $\bu$, let $|\bu|$ denote the vector derived
from $\bu$ by taking the absolute value of each of its components.
Starting as in the proof of Theorem~\ref{th:main}, using the 
dual norm property, we can write:
\begin{align*}
\h\cR_\sS(\sF_p) 
& =  \E_\bsigma\left[\sup_{\|\bw\|_p\leq W} \bw\cdot \sum_{i = 1}^m \sigma_i\bx_i
   \right]
\\
& = \frac{W}{m} \E_\bsigma\left[\left\| \left| \sum_{i = 1}^m \sigma_i\bx_i
 \right| \right\|_{p^*} \right]& (\text{dual norm property})\\
& \geq \frac{W}{m} \left\| \E_\bsigma\left[ \left| \sum_{i = 1}^m \sigma_i\bx_i
 \right| \right] \right\|_{p^*} 
& (\text{norm sub-additivity})\\
& = \frac{W}{m} \left[\sum_{j = 1}^d 
\left(\E_\bsigma \left[ \left| \sum_{i = 1}^m \sigma_i\bx_{ij}
  \right| \right]\right)^{p^*} \right]^{\frac{1}{p^*}} \\
& \geq \frac{W}{m} \left[\sum_{j = 1}^d 
\left(\frac{1}{\sqrt{2}} \left| \sum_{i = 1}^m \bx_{ij}^2
  \right|^{\frac{1}{2}} \right)^{p^*} \right]^{\frac{1}{p^*}} 
& (\text{Khintchine's ineq. \citep{Haagerup1981}})\\
& = \frac{W}{\sqrt{2} \, m} \left[\sum_{j = 1}^d 
\left[ \left| \sum_{i = 1}^m \bx_{ij}^2
  \right| \right]^{\frac{p^*}{2}} \right]^{\frac{1}{p^*}} \\
& = \frac{W}{\sqrt{2} \, m} \| \bX^\top \|_{2, p^*}.
\end{align*}
\end{proof}

\newpage
\section{Proof of Proposition~\ref{prop:norm_ratio}}
\label{sec:norm_comparison}

In this section, we prove Proposition~\ref{prop:norm_ratio}. 
This result implies that for $p \in (1, 2)$, the group norm
$\|\bX^\top\|_{2, p^*}$, is always a lower bound on the term
$\|\bX\|_{p^*,2}$ that appears in existing upper bounds. 
We first
present a simple lemma helpful for the proof.

\begin{lemma}
\label{lemma:norm_ratio} 
Let $1\leq p,r\leq \infty$ and $d$ be dimension. Then
\[
  \sup_{\| \bw\|_p\leq 1}\| \bw\|_{\du r}=\max(1, d^{1-\frac 1r -\frac
    1p})
\]
\end{lemma}

\begin{proof}
  We prove that, if $p\geq \du r$, then the following equality holds:
  \[
    \sup_{\| \bw\|_p\leq 1}\| \bw\|_{\du r}=d^{1-\frac 1r -\frac 1p},
  \]
  and otherwise that the following holds:
  \[
    \sup_{\| \bw\|_p\leq 1}\| \bw\|_{\du r}= 1.
  \]
  If $p\geq \du r$, by H{\"o}lder's generalized inequality with $\frac 1 {r^*} = \frac 1p + \frac 1s$,
  \[
    \sup_{\| \bw\|_p\leq 1} \| \bw\|_\du r \leq \sup_{\| \bw\| _p \leq 1} \| \one\|_s \| \bw\|_p = \| \one\|_s = d^\frac 1s = d^{\frac 1 {r^*} - \frac 1 p} = d^{1 - \frac 1r - \frac 1p}.
  \]
  Note that equality holds at the vector $\frac 1 {d^\frac 1p}\one$,
  and this implies that the inequality in the line above is an
  equality. Now for $p\leq \du r$, $\| \bw\|_p\geq \| \bw\|_\du r$,
  implying that $\sup_{\| \bw\|_p\leq 1}\| \bw\|_{\du r}\leq 1$. Here,
  equality is achieved at a unit vector $\be_1$.
\end{proof}

\noindent We now present the proof of Proposition~\ref{prop:norm_ratio}.

\propnormratioprop*

\begin{proof} 
First, \eqref{eq:p_leq_q} follows from \eqref{eq:q_leq_p} by
substituting $\bM=\bA^\top$ for a matrix $\bA$: For $q\leq p$,
\[
\min(m,d)^{\frac 1p -\frac 1q} \|\bA\|_{p,q}\leq \|\bA^\top\|_{q,p}\leq
\|\bA\|_{p,q}
\] 
which implies that 
\[
\|\bA^\top\|_{q,p}\leq \|\bA\|_{p, q}\leq \min(m,d)^{\frac 1q-\frac
  1p}\|\bA^\top\|_{q, p}
\]
However, now $p$ and $q$ are swapped in comparison to \eqref{eq:p_leq_q}. Now after swapping them again, for $p\leq q$,
\[
\|\bA^\top\|_{p,q}\leq \|\bA\|_{q,p}\leq \min(m,d)^{\frac 1p-\frac
  1q}\|\bA^\top\|_{p,q}
\]
The rest of this proof will be devoted to showing \eqref{eq:q_leq_p}.

Next, if $p=q$, then $\| \bM\|_{q,p}=\| {\bM^\top}\|_{p,q}$.  For the
rest of the proof, we will assume that $q < p$. Specifically,
$q < +\infty$ which allows us to consider fractions like $\frac pq$.

We will show that for $q < p$, the following inequality holds:
$\|\bM\|_{q, p} \leq \|\bM^\top\|_{p, q}$, or equivalently,
$\|\bM\|^q_{q, p} \leq \|\bM^\top\|^q_{p, q}$.

We will use the shorthand $r = \tfrac{p}{q} > 1$.
By definition of the group norm and using the notation $\bU_{ij} =
|\bM_{ij}|^p$, we can write
\begin{align*}
  \| \bM \|^q_{q, p}
  = \bigg[ \sum_{i = 1}^m \Big[\sum_{j = 1}^d |\bM_{ij}|^q
  \Big]^{\frac{p}{q}} \bigg]^{\frac{q}{p}}
  = \bigg[ \sum_{i = 1}^m \Big[ \sum_{j = 1}^d\bU_{ij}
  \Big]^{r} \bigg]^{\frac{1}{r}} \mspace{-10mu}
  & = \left\| 
  \left[
  \begin{smallmatrix}
    \sum_{j = 1}^d \bU_{1j}\\
    \vdots\\
    \sum_{j = 1}^d \bU_{mj}
  \end{smallmatrix}
  \right]
  \right\|_{r}\\
& \leq \sum_{j = 1}^d \left\| \left[
  \begin{smallmatrix}
    \bU_{1j}\\
    \vdots\\
    \bU_{mj}
  \end{smallmatrix}
  \right] \right\|_r
  = \sum_{j = 1}^d \Big[  \sum_{i = 1}^m|\bM_{ij}|^p
  \Big]^{\frac{q}{p}}
  = \|\bM^\top\|^q_{p, q}.
\end{align*}
To show that this inequality is tight, note that equality holds for an all-ones matrix.
Next, we prove the inequality 
\[
  \min(m, d)^{\frac 1q - \frac 1p} \|\bM^\top\|_{p, q}
  \leq \|\bM\|_{q, p},
\] 
for $q \leq p$.
Applying Lemma~\ref{lemma:norm_ratio} twice gives
\begin{equation}
  \label{eq:lin_app_d} \|\bM^\top \|_{p,q}\leq \| \bM^\top\|_{q,q}=\|\bM\|_{q,q}\leq d^{\frac 1q-\frac 1p}\|\bM\|_{p,q}.
\end{equation}
Again applying Lemma~\ref{lemma:norm_ratio} twice gives
\begin{equation}
  \label{eq:lin_app_m} \|\bM^\top\|_{p,q}\leq m^{\frac 1q-\frac 1p}\|\bM^\top\|_{p,p}=m^{\frac 1q-\frac 1p}\|\bM\|_{p,p}\leq m^{\frac 1q-\frac 1p}\|\bM\|_{p,q}.
\end{equation}
Next, we show that \eqref{eq:lin_app_d} is tight if $d\leq m$ and that \eqref{eq:lin_app_m} is tight if $d \geq m$. If $d\leq m$, the bound is tight for the block matrix
$\bM 
= \left[
\begin{smallmatrix} 
\bI_{d \times d} \ | \ \mathbf 0
\end{smallmatrix}
\right]$,
and, if 
$d\geq m$, then the bound is tight for the block matrix
$\bM 
= \left[
\begin{smallmatrix} \bI_{d\times d} \\[.075cm]
\hline\\
\mathbf 0
\end{smallmatrix}
\right].$
\end{proof}

\newpage
\section{Proof of  Theorem~\ref{th:rc_linear_comparison}}
\label{app:compare}

\normratioprop*

Both Theorem~\ref{th:main} and equation \eqref{eq:previous_linear_rc}
present upper bounds on $\h\R_{\sS}(\sF_p)$ for $1 < p \leq 2$. Both
of these bounds are of the form a constant times a matrix norm of
$\bX$. In Appendix~\ref{sec:norm_comparison}, we compared the two
matrix norms and proved the inequality
$\|\bX^\top\|_{2, p^*}\leq \|\bX\|_{p^*,2}$ in the relevant region
(Lemma~\ref{lemma:norm_ratio}). Here, we compare the two constants and
show that the constant associated with Theorem~\ref{th:main} is
smaller than the one appearing in \eqref{eq:previous_linear_rc}
(Lemma~\ref{lemma:constant-comparison}).  These lemmas combined
directly prove Theorem~\ref{th:rc_linear_comparison}.

In this section, we study the constants in the two known bounds on the
Rademacher complexity of linear classes for $1 < p \leq
2$. Specifically,
\begin{numcases}{\h\R_\sS(\sF_p)\leq}
\frac W m \sqrt{p^*-1} \|{\bX}\|_{p^*,2} 
& \label{eq:previous_linear_rc_2}\\[.25cm] 
\frac{\sqrt{2}W}{m} \bigg[\frac{\Gamma( \tfrac{p^* + 1}{2} )}{\sqrt{\pi}} \bigg]^{\frac{1}{p^*}} \|{\bX^\top}\|_{2, p^*} &  
\label{eq:new_linear_rc}
\end{numcases}
We will compare the constants in equations
(\ref{eq:previous_linear_rc_2}) and (\ref{eq:new_linear_rc}), namely
$\frac{\sqrt 2 W}m \big(\frac{\Gamma(\frac{p^*+1} 2 )}{\sqrt
  \pi}\big)^\frac 1 {p^*} $ and $\frac W m\sqrt{p^*-1}$. Since
$\frac Wm$ divides both of these constants, we drop this factor and
work with the expressions $c_1(p) \colon = \sqrt{p^* - 1}$ and
$c_2(p)\colon = \sqrt 2 \big(\frac{\Gamma(\frac{p^*+1} 2 )}{\sqrt
  \pi}\big)^\frac 1 {p^*}$.

Here we establish our main claim that $c_2(p) \leq c_1(p)$.
\begin{lemma}
\label{lemma:constant-comparison}
Let $c_1(p)=\sqrt{p^*-1}$ and $c_2(p)=\sqrt 2 \big(\frac{\Gamma(\frac{p^*+1} 2 )}{\sqrt \pi}\big)^\frac 1 {p^*}$.
Then
\[
c_2(p) \leq c_1(p),
\]
for all $1 \leq p \leq 2$.
\end{lemma}
\begin{proof}
First note that $c_1(2)=c_2(2)$. For convenience, set $q=p^*$, $f_1(q)=c_1(p)$, and $f_2(q)=c_2(p)$. We claim $\frac d {dq} f_1(q)\geq \frac d {dq} f_2(q)$ for $q\geq 2$, and this implies that $c_2(p)\leq c_1(p)$ for $1\leq p\leq 2$. 

The rest of this proof is devoted to showing that $\frac d {dq} f_1(q)\geq \frac d {dq} f_2(q)$. Upon differentiating we get that $f_1'(q)=\frac 1 {2\sqrt{q-1}}$. Next, we will differentiate $f_2$. To start, we state a useful inequality~(see Equation~$2.2$ in \cite{Alzer1997})
bounding the digamma function, $\psi(x)$. 
\begin{align}
\psi(x)\leq \log(x)-\frac 1{2x}
\label{eq:digamma_bound}
\end{align}
Recall that the digamma function is the logarithmic derivative of the
gamma function, $\psi(x) = \frac d {dx} (\log\Gamma(x))=\frac{\Gamma'(x)}{\Gamma(x)}$.
Now we differentiate $\ln f_2$:
\begin{align*}
\frac d {dq} (\ln f_2(q)) 
& =\frac{ \frac q2 \psi(\frac {q+1} 2)-(\ln(\Gamma(\frac {q+1} 2))-\ln(\sqrt \pi))}{q^2}\\
& \leq\frac{ \frac q2(\log(\frac {q+1}2-\frac 1
  {q+1})-(\ln(\Gamma(\frac{q+1}2))-\ln\sqrt \pi)}{q^2} &\,\text{(by (\ref{eq:digamma_bound}))}\\
& \leq \frac {\frac q2(\log \frac {q+1}2 -\frac 1 {q+1})-(\frac 12 \ln2 +\frac q2\log \frac {q+1}2-\frac {q+1}2)}{q^2}&\text{(by the left-hand equality in (\ref{eq:gamma_bound}))}\\
& = \frac 1 {2q}+\frac 1 {q^2}\Big(\frac 1 {2(q+1)}-\frac 12\log2\Big)\\
& \leq \frac 1 {2q}.
\end{align*}
The last line follows since we only consider $q\geq 2$ and $\frac 1 {2(q+1)}-\frac 12\ln 2\leq 0$ in this range.
Finally, the fact that $\frac d {dq}(\ln f_2(q))=f_2'(q)/f_2(q)$ implies 
\begin{align*}
f_2'(q) 
& = f_2(q) \frac d {dq} (\ln f_2(q))\\
& \leq \frac 1 {2q} f_2(q)&\,\text{(by }\frac d {dq}(\ln f_2(q))\leq \frac 1 {2q} \text{)}\\
& \leq \frac{e^{-\frac 12} \sqrt{q+1}}{2q}&\,\text{(by applying the upper bound in Lemma~\ref{lemma:f2_bound})}\\
& =\frac 1 {2\sqrt{q-1}}\frac{e^{-\frac 12}\sqrt{(q+1)(q-1)}}q\\
& \leq e^{-\frac 12} \frac 1{2\sqrt{q-1}}&\,\text{(using }q^2-1\leq q^2)\\
& \leq \frac 1 {2\sqrt{q-1}}=f_1'(q)&\,(\text{using }e^{-\frac 12}<1).
\end{align*} 
\end{proof}
\section{The Tightness of the $\sqrt{\log(d)}$ factor for $p=1$}\label{app:sqrt_log_d}
Here, we provide an example showing that the dimension dependence of $\sqrt{\log(d)}$ in our upper bound on the Rademacher complexity of linear functions bounded in $\ell_1$ norm is tight.

Consider a data set with $d=2^m$. Then the data matrix $\bX$ has $2^m$ rows. We pick the data $\{\bx_i\}$ so that the rows of $\bX$ are the set $\{-1,+1\}^m$. This means that $\|\bX^\top\|_{2, p^*}=\sqrt m$ and we can compute the Rademacher complexity as
\begin{align*}
\h \R_\sS(\sF_1)&=
\frac 1m \E_\bsigma \left[ \sup_{\|\bw\|_1
\leq W}\bw \cdot \sum_{i=1}^m \sigma_i \bx_i\right]=
\frac 1m \E_\bsigma \left[ \left\|\sum_{i=1}^m \sigma_i \bx_i\right\|_\infty \right]&\text{(definition of dual norm)}\\
&=\frac 1m \E_\bsigma\left[ \max_{1\leq j\leq d} \sum_{i=1}^m \sigma_i (\bx_i)_j\right]=
\frac 1m \E_\bsigma\left[ m\right]&\text{(tightness of Cauchy-Schwartz)}\\
&=\frac m m=\frac 1m \sqrt m \sqrt m= \frac 1m \sqrt{\log (d)} \|\bX^\top\|_{2,\infty} &(d=2^m,\|\bX^\top\|_{2,\infty}=\sqrt m) \end{align*} Therefore, the $\sqrt{\log(d)}$ dependence in the constant for $p=1$ is tight.
\end{document}